\documentclass{article}

\usepackage{PRIMEarxiv}

\usepackage[pdftex]{graphicx}
\DeclareGraphicsExtensions{.pdf,.jpeg,.png}

\usepackage{amsmath}
\usepackage{amssymb} 
\usepackage[colorlinks=true,
			urlcolor=black,
			citecolor=black,
			linkcolor=black	]{hyperref}

\newcommand{\figwidth}{.95\linewidth}

\usepackage{amsthm}
\newtheorem{theorem}{Theorem}
\newtheorem{remark}{Remark}

\usepackage{bm}
\newcommand{\q}{\boldsymbol{q}}
\renewcommand{\u}{\boldsymbol{u}}
\newcommand{\ud}{\boldsymbol{u}_\mathrm{nom}}
\newcommand{\x}{\boldsymbol{x}}
\newcommand{\f}{\boldsymbol{f}}
\newcommand{\g}{\boldsymbol{g}}
\newcommand{\te}{\boldsymbol{\tau}_\mathrm{ext}}
\newcommand{\teh}{\hat{\boldsymbol{\tau}}_\mathrm{ext}}
\newcommand{\D}{\boldsymbol{D}(\q)}
\newcommand{\C}{\boldsymbol{C}(\q, \dot{\q})}
\newcommand{\G}{\boldsymbol{g}(\q)}
\newcommand{\B}{\boldsymbol{B}}
\newcommand{\Pe}{P_\mathrm{ext}}
\newcommand{\Km}{K_\mathrm{max}}
\newcommand{\Ke}{K_\mathrm{e}}
\newcommand{\Psid}{\Psi(\q, \dot{\q}, \ud)}
\DeclareMathOperator*{\argmin}{argmin}

\usepackage{subcaption}
\usepackage[font=small,labelfont=bf]{caption}

\usepackage{enumitem}

\usepackage[capitalize,nameinlink]{cleveref}
\crefformat{section}{#2Sec.~#1#3} 
\crefmultiformat{section}{Sec.~#2#1#3}{ and~#2#1#3}{, #2#1#3}{ and~#2#1#3}

\usepackage{fullwidth}
\usepackage{fancyhdr}  

\title{
	Limiting Kinetic Energy through Control Barrier Functions: Analysis and Experimental Validation
}

\author{
Federico Califano$^1$, Daniël Logmans, and Wesley Roozing$^1$ \\
$^1$ Robotics \& Mechatronics (RaM), University of Twente, The Netherlands.\\
\texttt{Contacts:\{w.roozing,f.califano\}@utwente.nl.}\\
}

\begin{document}
\maketitle

\begin{abstract}
In the context of safety-critical control, we propose and analyse the use of Control Barrier Functions (CBFs) to limit the kinetic energy of torque-controlled robots. The proposed scheme is able to modify a nominal control action in a minimally invasive manner to achieve the desired kinetic energy limit. We show how this safety condition is achieved by appropriately injecting damping in the underlying robot dynamics independently of the nominal controller structure. We present an extensive experimental validation of the approach on a 7-Degree of Freedom (DoF) Franka Emika Panda robot. The results demonstrate that this approach provides an effective, minimally invasive safety layer that is straightforward to implement and is robust in real experiments. A video of the experiments can be found \href{https://youtu.be/3fZdLql6-CE}{here}
\end{abstract}

\twocolumn

\section{Introduction}
\label{sec:introduction}


Collaborative robots, sometimes called \textit{cobots}, are gaining traction across a wide range of industries, including logistics, service robotics, and manufacturing \cite{matheson2019human,robotics12030084}. Safety is a critical control objective when these robots share space with humans \cite{Robla-Gomez2017WorkingEnvironments,Hamad2023AInteraction}. The recent rise of learning-based controllers, which typically only provide probabilistic safety guarantees, underscores the need for safety-critical approaches \cite{Brunke2022SafeLearning}. ISO standards \cite{InternationalOrganizationforStandardization2016Robots15066:2016} attempt to formalise the safety hazards in this setting, and their mitigation is an active research area. Some works prevent interaction, by enforcing a speed-dependent separation distance between the robot and operator, assuming reliable detection methods \cite{Malm2019DynamicRobots,mansfeld_safety_2018}. Other works limit long-duration interaction power and force by implementing, e.g., impedance control \cite{Kishi2003, Tsetserokou2007, roozing_energy-based_2020}. Yet other approaches explore various dynamic human-robot impact scenarios and relate the impact velocity to the risk of injury \cite{Haddadin2016}.

In this work safety is addressed by bounding the kinetic energy that could potentially be transferred to a human operator, in order to prevent injury in collision scenarios. The importance of this choice is backed by numerous publication which relate directly relevant safety metrics to the energy flow generated from the interaction \cite{Tadele2014ThePublications,Hamad2023AInteraction}. Furthermore, the \textit{power and force limiting} (PFL) conditions in the ISO/TS 15066 \cite{InternationalOrganizationforStandardization2016Robots15066:2016}, which are the only collaborative conditions in which contact between humans and robots are considered, are addressed through energetic constraints.

We propose a method that takes the form of a \textit{safety filter}, enforcing a bound on maximum kinetic energy while minimally altering a desired control input. We make use of \textit{Control Barrier Functions} (CBFs), a safety-critical control algorithm able to constrain the robot to a region of its state space representing safe operating conditions \cite{Ames2019}. Most CBF implementations in robotics apply to safety-critical kinematic control (i.e., tasks in which the safety constraint represents obstacle avoidance conditions) and rely on lower-level controllers to handle system dynamics \cite{Cortez2020, Landi2019, Rauscher2016ConstrainedFunctions}. Instead, we investigate the use of energy-based CBFs and, different from previous works such as \cite{Ferraguti2022,Singletary2021}, utilize them to directly limit the kinetic energy of a torque-controlled robot. 

We recognise relevant related works proposing schemes to limit the kinetic or total energy of torque-controlled manipulators for safety objectives. These works are motivated by \textit{energy-aware} and \textit{passivity} arguments \cite{Capelli2022,ortega,ear}, stressing the fact that safety measures are closely related to energy- and power-based metrics.
The recent work \cite{Michel2024} presents a control algorithm that is able to limit the kinetic energy,  achieved by using higher-order CBFs in a system augmented with energy tanks \cite{Califano2022OnSystems}, used to enforce passivity of the overall scheme. Other approaches attempt to limit kinetic energy \cite{BENZI20231288} and total energy \cite{Raiola2018DevelopmentRobots} of controlled robots, also using energy-tank based arguments to recover passivity.

In this work, we present a novel approach that avoids considering passivity as a strict constraint to be achieved at design phase. Instead, we achieve the kinetic energy bound directly through the proposed CBF-based algorithm. We analytically and experimentally show that the proposed CBF operates solely by injecting damping into the system, ensuring that the safety-critical control action inherently preserves the passivity of any nominal passive closed-loop system. This eliminates the need for supplementary tools such as energy tanks, making the proposed scheme significantly simpler than most of the state-of-the-art solutions.

The main contributions of this paper are:
\begin{enumerate}
    \item A kinetic energy-limiting CBF-based safety filter and analysis of its energetic properties.
    \item Extensive experimental validation on a 7-DoF robot manipulator of the proposed safety-critical control system.
\end{enumerate}
The remainder of this paper is outlined as follows. \cref{sec:methods} contains the mathematical background and analysis involving the specific CBF used in this work. We present extensive experimental results in four scenarios in \cref{sec:experiments}. Finally, \cref{sec:discussion,sec:conclusions} discuss the results and conclude the paper.

\section{Methods}
\label{sec:methods}

In this section, we introduce the CBF control approach (we refer to \cite{Ames2019,Ames2017,XU201554} and references therein for a detailed overview), followed by the specific CBF used in this work.

We denote vectors, matrices, vector-valued maps and matrix-valued maps with bold font letters; scalars and real-valued maps with regular letters; and sets with calligraphic letters. The norm of a vector is denoted by $||\cdot||$.

\subsection{Background}\label{sec:VanillaCBF}
\subsubsection{Control Barrier Functions}
Consider a nonlinear control-affine system in standard form
\begin{equation}\label{eq:affinecontrolsystem}
    \dot{\x}=\f(\x)+\g(\x)\u
\end{equation}
with system state $\x \in \mathcal{D} \subset \mathbb{R}^n$ and control input $\u \in \mathcal{U} \subset \mathbb{R}^m$. All variables are assumed to have a degree of continuity such that the right-hand side of \eqref{eq:affinecontrolsystem} is locally Lipschitz, to guarantee the existence and uniqueness of the solutions.

Control barrier functions (CBFs) serve to achieve \textit{forward invariance} of a set $\mathcal{S}$, referred to as \textit{safe set}, i.e.,
\begin{equation}
    \forall \x(0) \in \mathcal{S} \implies \x(t)\in \mathcal{S} \,\,\, \forall t>0.
\end{equation}
The safe set $\mathcal{S}$ is built as the superlevel set of a continuously differentiable function $h:\mathcal{D}\to \mathbb{R}$, i.e., \[
    \mathcal{S} = \{ \x\in \mathcal{D} : h(\x)\geq0 \}.\]
The function $h(\x)$ is then defined as a CBF on $\mathcal{D}$ if $\partial_{\x} h(\x)\neq 0, \forall \x \in \partial \mathcal{S}$ and
\begin{equation}\label{eq: CBF}
    \sup_{\u \in \mathcal{U}} \underbrace{\left[ \frac{\partial h}{\partial \x} \f(\x)+\frac{\partial h}{\partial{\x}}\g(\x)\u \right] 
    }_{\dot{h}(\x, \u)} 
    \ge -\alpha(h(\x))
\end{equation}
for all $\x\in \mathcal{D}$ and some \textit{extended class $\mathcal{K}$ function}\footnote{A function $\alpha: (-b,a) \to (- \infty, \infty)$ with $a,b>0$, which is continuous, strictly increasing, and $\alpha(0)=0$.} $\alpha$. Here we denote the term in the square bracket, i.e., the variation of $h$ along the solution of \eqref{eq:affinecontrolsystem}, by $\dot{h}(\x,\u)$.

The link between the existence of a CBF and the forward invariance of the related safe set is established by the following key result.

\begin{theorem}[\cite{Ames2017}]
    \label{thm:1}
    Let $h(\x)$ be a CBF on $\mathcal{D}$ for \eqref{eq:affinecontrolsystem}. Any locally Lipschitz controller $\u=\boldsymbol{k}(\x)$ such that $\dot{h}(\x,\u) \geq -\alpha (h(\x))$ provides forward invariance of the safe set $\mathcal{S}$. Additionally the set $\mathcal{S}$ is asymptotically stable on $\mathcal{D}$.
\end{theorem}

The way controller synthesis induced by CBFs are implemented in practice is to use them as \textit{safety filters} (\cref{fig:safety_filter}), transforming a nominal state-feedback control input $\ud(\x)$ into a new state-feedback control input $\u(\x)$ in a minimally invasive fashion in order to guarantee forward invariance of $\mathcal{C}$. In practice, the following Quadratic Program (QP) is solved:
\begin{equation}
    \label{eq:CBFcontrol}
    \begin{aligned}
        \u(\x)= & \argmin_{\u\in \mathcal{U}} \quad  ||\u-\ud(\x) ||^2\\
         & \mathrm{s.t.}  \quad  \dot{h}(\x,\u) \geq -\alpha(h(\x)) 
    \end{aligned}
\end{equation}
The transformation of the nominal control input $\ud(\x)$ into the new state-dependent control $\u(\x)$ by solving \eqref{eq:CBFcontrol} is referred to as \textit{safety-critical control}, and the fact that the constraint is linear in the input allows for efficient real-time implementations of such a controller.
In this work we will take advantage of the additive decomposition of the safety-critical control:
\begin{equation}\label{eq:safetycriticaldecomposition}
    \u(\x)=\ud(\x)+\u_{\mathrm{safe}}(\x),
\end{equation}
resulting from the solution of \eqref{eq:CBFcontrol}, where clearly $\u_{\mathrm{safe}}(\x)=0$ when $\dot{h}(\x,\u) \geq -\alpha(h(\x))$.

\begin{figure}[ht]
	\centering
	\includegraphics[width=0.75\linewidth]{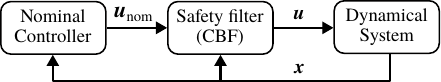}
	\caption{CBF-based safety filter.}
	\label{fig:safety_filter}
\end{figure}


\subsubsection{Robots and their energetic analysis}
In this work we will consider fully actuated $n$-Degree-of-Freedom torque-controlled robots, represented as a $2n$-dimensional system whose state belongs to the tangent bundle of the robot's configuration manifold, expressed in the usual Lagrangian canonical coordinates $\q \in \mathbb{R}^n$ and $\dot{\q} \in \mathbb{R}^n$.
The system dynamics are expressed by the Euler-Lagrange equations:
\begin{equation} \label{eq:eulerlagrangesystem}
    \D\ddot{\q}+\C\dot{\q}+\G=\B\u,
\end{equation}
with $\D$ the inertia matrix, $\C$ the Coriolis matrix, $\G$ the gravity vector, and $\B$ the full rank actuation matrix. System \eqref{eq:eulerlagrangesystem} admits a representation as a control affine system \eqref{eq:affinecontrolsystem} when choosing $\x = (\q^{\top} ,\dot{\q}^{\top} )^{\top} \in \mathbb{R}^{2n}$, and as such CBF-based algorithms can be applied.

In the sequel we will perform analysis involving the kinetic energy:
\begin{equation}
    \Ke(\q, \dot{\q})=\frac{1}{2}\dot{\q}^{\top}\D(\q) \dot{\q}.
\end{equation}
Using \eqref{eq:eulerlagrangesystem} and the skew symmetry of the matrix $\dot{\boldsymbol{D}}-2\boldsymbol{C}$ (see e.g., \cite{Singletary2021}), it is straightforward to verify that the rate of change of the kinetic energy along solutions of \eqref{eq:eulerlagrangesystem} verifies:
\begin{equation}\label{eq:variationK}
    \dot{K}_\mathrm{e}(\q, \dot{\q})=-\dot{\q}^{\top}\G+\dot{\q}^{\top}\B \u,
\end{equation}
where the last term $\dot{\q}^{\top}\B \u$ represents the instantaneous mechanical power that the controller injects into the robot.

\subsection{Bounding kinetic energy with CBFs}
Since we aim to set an upper bound $\Km$ on the robot's kinetic energy, we propose to encode the safe set $\mathcal{S}=\{(\q^{\top},\dot{\q}^{\top} )^{\top} \in \mathbb{R}^{2n} : \Ke(\q, \dot{\q}) \leq \Km\}$ through the CBF 
\begin{equation}\label{eq:ECBF}
    h(\q,\dot{\q}) = \Km -\Ke(\q,\dot{\q}) = \Km-\frac{1}{2}\dot{\q}^\top \D\dot{\q}.
\end{equation}
Using \eqref{eq:variationK}, it immediately follows that
\begin{equation}\label{eq:CBFdot}
    \dot{h}(\q,\dot{\q},\u)=\dot{\q}^{\top}\G-\dot{\q}^{\top}\B \u.
\end{equation}
If \eqref{eq:ECBF} is a valid CBF for the controlled robot, then \cref{thm:1} guarantees that, if we are able to solve \eqref{eq:CBFcontrol}, the kinetic energy limit is always respected. 

For the following analysis, let us define the quantity 
\begin{equation}\label{eq:Psi}
    \Psi(\q, \dot{\q}, \u) =\dot{h}(\q,\dot{\q},\u)+\alpha(h(\q,\dot{\q})),
\end{equation}
so that the constraint in \eqref{eq:CBFcontrol} reads $\Psi(\q, \dot{\q}, \u)\geq 0$, or, when expanded using \eqref{eq:ECBF} and \eqref{eq:CBFdot}:
\begin{equation} \label{eq:ECBFconstraint}
    -\dot{\q}^\top \B \u + \dot{\q}^\top \G
    \ge -\alpha\left(\Km - \frac{1}{2}\dot{\q}^\top \D \dot{\q}\right).
\end{equation}
Notice that the left hand side of this expression is the negative of the net power flowing into the robot, due to gravitational effects and due to the control $\u$. 

The following result states an interesting fact: when using a CBF in the form \eqref{eq:ECBF}, the safety-critical component $\u_{\mathrm{safe}}$ in the decomposition \eqref{eq:safetycriticaldecomposition} only injects negative power into the underlying closed-loop system \eqref{eq:eulerlagrangesystem} controlled with $\ud$.

\begin{theorem} \label{th:dampingonly}
    Let \eqref{eq:ECBF} be the CBF acting on a system \eqref{eq:eulerlagrangesystem}, controlled with nominal input $\ud(\q,\dot{\q})$ and resulting in the safety-critical decomposition \eqref{eq:safetycriticaldecomposition}.
    The total power injected by the safety filter is always non-positive, that is,
    \begin{equation} \label{eq:dampingonly}
        P_{\mathrm{safe}} := \dot{\q}^\top \B \u_\mathrm{safe} \le 0.
    \end{equation}
\end{theorem}

\begin{proof}
    We distinguish two cases. First, if $\Psi(\q, \dot{\q}, \ud) \ge 0$, the safety constraint in \eqref{eq:CBFcontrol} is satisfied with the trivial solution $\u = \ud$, and as such $\u_{\mathrm{safe}}=0$ and $P_{\mathrm{safe}}=0$. 
    
    Secondly, we consider the case $\Psi(\q, \dot{\q}, \ud) < 0$.
    We need to show that \eqref{eq:dampingonly} holds, which can be rewritten using \eqref{eq:safetycriticaldecomposition} as
    \begin{equation}\label{eq:powerproof1}
        \dot{\q}^\top \B \u \le \dot{\q}^\top \B \ud.
    \end{equation}
    The CBF enforces the inequality \eqref{eq:ECBFconstraint}, which is rewritten as:
    \begin{equation} \label{eq:powerproofoutside}
        \dot{\q}^\top \B \u \le \alpha(h(\q, \dot{\q}))  + \dot{\q}^\top \G.
    \end{equation}
    The task is now to show that \eqref{eq:powerproofoutside} implies \eqref{eq:powerproof1}. A sufficient condition for the latter proposition to be true is that:
    \begin{equation}
        \alpha(h(\q, \dot{\q}))+\dot{\q}^\top \G \le \dot{\q}^\top \B \ud,
    \end{equation}
    which can be rewritten as
    \begin{equation} \label{eq:powerproof2}
        \dot{h}(\q, \dot{\q},\ud) + \alpha(h(\q, \dot{\q}))
        = \Psi(\q, \dot{\q}, \ud)
        \le 0,
    \end{equation}
    which is true by assumption.
    %
    %
    %
\end{proof}

It is worth noticing that the previous result holds independently on the specific control law $\ud(\q,\dot{\q})$: whatever design is chosen for the nominal controller, if a CBF in the form (\ref{eq:ECBF}) is used, the safety-critical component of the closed-loop system $\u_{\mathrm{safe}}$ will always act in a way to extract mechanical energy from the system.

\begin{remark}[Relation to energy-based CBFs in \cite{Singletary2021}]
    The CBF \eqref{eq:ECBF} shares some properties with so called energy-based CBFs as introduced in \cite{Singletary2021} for safety-critical kinematic control, defined as $h(\q,\dot{\q})=-K_e(\q,\dot{\q})+\beta \bar{h}(\q)$ with $\beta>0$. In particular the proposed CBF (\ref{eq:ECBF}) shares a technical advantage with the one in \cite{Singletary2021}: the safety constraint \eqref{eq:ECBFconstraint} is independent of the Coriolis matrix, reducing model dependence and computational complexity while solving \eqref{eq:CBFcontrol}.
\end{remark}

\begin{remark}[Passivity/Stability preservation property of (\ref{eq:ECBF})]
    A convenient consequence of a negative power injection by the safety-critical control component $\dot{\q}^{\top} \B \, \u_{\mathrm{safe}}\leq 0$ is the following: if the controlled system  with $\ud(\q,\dot{\q})$ is passive (or stable), then the safety-critical control preserves the closed-loop passivity (or stability) properties of the nominal controller. This fact allows for assessing passivity of the critically controlled closed-loop system without the use of extra passivising framework such as energy tanks as done in e.g., \cite{Raiola2018DevelopmentRobots,Michel2024}.
    In \cite{Califano2023} a variation of this result (with different generality and different proof) was indeed given in the framework of passivity-based control.\end{remark}

Another important result in the CBF framework is that, under the conditions of \cref{th:dampingonly}, together with the extra assumption that $\mathcal{U}=\mathbb{R}^m$, the solution \eqref{eq:safetycriticaldecomposition} assumes the analytic expression:
\begin{equation}\label{eq:closedFormCBF}
 \u_{\mathrm{safe}}=   \begin{cases}
        \frac{\B^\top \dot{\q}}{||\B^\top \dot{\q}||^2} \Psid & \text{if } \Psid < 0\,,\\
        0           & \text{otherwise.}
    \end{cases}
\end{equation}
This expression is useful since it induces an analytic expression for the power injected in the system by the controller, clearly displaying the role of the function $\alpha(\cdot)$ in the CBF algorithm. For example, it is simple to see that using \eqref{eq:closedFormCBF} the expression \eqref{eq:dampingonly} becomes $P_{\mathrm{safe}}=\Psid$ (when $\Psid<0$), providing a quantitative measure on how much damping the safety-critical control injects into the system and how it can be modulated with $\alpha(\cdot)$. In the next section we will use \eqref{eq:closedFormCBF} to perform a thorough power analysis involving unmodelled external interactions.

\subsection{External interaction forces}
The dynamic system \eqref{eq:eulerlagrangesystem} does not include external torques caused by disturbances or interaction. For an external generalised force vector $\te\in \mathbb{R}^n$, the system becomes:
\begin{equation}
    \label{eq:eulerlagrangesysteminteractive}
    \D\ddot{\q}+\C\dot{\q}+\G=\B\u + \te
\end{equation}
where typically the external torques are expressed as $\te=\boldsymbol{J}(\q)^{\top}\textbf{f}$, where $\textbf{f}$ are interaction forces applied at the end-effector of the robot and $\boldsymbol{J}(\q)$ denotes the end-effector Jacobian matrix.

Obtaining these external interaction forces and including them in the model leads to the substitution of \eqref{eq:ECBFconstraint} with the new constraint:
\begin{equation}
    \label{eq:ECBFconstraintinteractive}
    \underbrace{\dot{\q}^\top (-\B \u -\te + \G)}_{\dot{h}(\q,\dot{\q}, \u)} 
    \ge -\alpha\left(\Km - \frac{1}{2}\dot{\q}^\top \D \dot{\q}\right).
\end{equation}
The challenge lies in measuring $\te$, or, in practice, finding a good estimate $\teh$. If accurately estimated, the previously presented control scheme using (\ref{eq:ECBF}) and \eqref{eq:ECBFconstraintinteractive} will keep the kinetic energy below the desired limit. We refer to this case in the sequel as \textit{interaction-aware}.
If these torques are \textit{not} taken into account in the model (i.e., the the CBF algorithm is implemented with the constraint \eqref{eq:ECBFconstraint}), system invariance cannot in general be guaranteed. We refer to this case in the sequel as \textit{interaction-agnostic}. In the following we state a result that gives insight on the behaviour of the controlled system in interaction-agnostic case.

\begin{theorem}\label{th:povershoot}
    Consider system \eqref{eq:eulerlagrangesysteminteractive} with unknown external torques $\te$ producing a positive power inflow $\Pe=\dot{\q}^{\top}\te\geq 0$. Let the system \eqref{eq:eulerlagrangesysteminteractive} be controlled with the CBF \eqref{eq:ECBF} and assume $\mathcal{U}=\mathbb{R}^m$. If the system converges to a positive constant kinetic energy value with $\u_{\mathrm{safe}}\neq 0$, the kinetic energy error $\Ke-\Km$ converges to the relation:
    \begin{equation}\label{eq:powerovershoot}
        \Ke-\Km = \alpha^{-1}(\Pe).
    \end{equation}
\end{theorem}

\begin{proof}
    The variation of kinetic energy of \eqref{eq:eulerlagrangesysteminteractive} results in:
    \begin{equation}
        \dot{K}_\mathrm{e}(\q, \dot{\q})=-\dot{\q}^{\top}\G+\dot{\q}^{\top}\B \u +\Pe,
    \end{equation}
    where $\Pe>0$ contributes to a transient increase of kinetic energy. 
    Imposing the steady-state condition $\dot{K}_\mathrm{e}=0$, substituting \eqref{eq:closedFormCBF} in the case $\Psi(\q,\dot{\q},\ud)<0$ (because $\u_{\mathrm{safe}}\neq 0)$ and using the assumption of a non-zero steady-state kinetic energy (i.e., $\dot{\q} \neq 0$) one obtains:
    \begin{equation}
      - \dot{\q}^\top \G + \dot{\q}^{\top}\B \ud + \Psi(\q,\dot{\q},\ud) +\Pe = 0.
    \end{equation}
    Now, using \eqref{eq:CBFdot} and \eqref{eq:Psi}, the power balance simplifies to\footnote{It is worth noticing that \eqref{eq:steadystatepower} does not depend on $\ud$, even if the control action \eqref{eq:closedFormCBF} does.}
    \begin{equation}
    \label{eq:steadystatepower}
        \alpha(h(\q,\dot{\q}))=-\Pe,
    \end{equation}
    from which \eqref{eq:powerovershoot} follows, due to the invertibility of the class $\mathcal{K}$ function $\alpha(\cdot)$.
\end{proof}

\begin{remark} \label{rem:3}
    The last result gives further insight on the role of the class $\mathcal{K}$ function $\alpha$. For example, as easily demonstrated using a linear function $\alpha(h)=\gamma h$ with $\gamma>0$, an increase in $\gamma$ reduces the kinetic energy error that is incurred due to unmodelled power flows, as $\Ke-\Km \leq \Pe/ \gamma$. At a design stage, this fact needs to be traded with the advantages of choosing a lower value of $\gamma$. Intuitively, lowering $\gamma$ induces a more conservative behaviour to achieve invariance in nominal conditions, tending to push the system state towards the safe set before reaching its boundary (see e.g., \cite{Ames2019}). As a consequence, lowering $\gamma$ corresponds to smoother closed-loop behaviour in nominal conditions, but also to poorer rejection in case of unmodelled external disturbances.
\end{remark}
\section{Experimental results}
\label{sec:experiments}
To assess the practical applicability of the approach presented in this paper in the context of imperfect torque tracking, sensor errors, discrete-time control, and communication delays, we present extensive experimental results.
This comprises four different experiments:
\begin{enumerate}
    \item A Cartesian step response;
    \item Contact loss with the environment;
    \item External power input by human-robot interaction;
    \item Kinetic energy error validation (\cref{th:povershoot}).
\end{enumerate}

The nominal control action for the first two experiments is generated by an underdamped Cartesian impedance controller with a stiffness of 200\,N/m and a damping of 6\,N\,s\,m\textsuperscript{-1}. For the latter two experiments, the nominal controller is deactivated\footnote{Note that gravity and friction compensation are still present as part of the lower-level Franka Control Interface (FCI).}.
For all experiments, we choose $\alpha(h)=\gamma h$, and we will investigate its influence: when decreasing $\gamma$, we expect increasingly conservative behaviour (see \cref{rem:3}).

\subsection{Experimental Setup}


Experiments are performed on a Franka Emika Panda 7-DoF robotic arm. The Franka Control Interface (FCI) provides a ROS-interface for joint torque commands at 1000\,Hz, with built-in gravity and friction compensation active by default. As a result, we set $\G = \mathbf{0}$ in \cref{eq:eulerlagrangesystem,eq:eulerlagrangesysteminteractive}.
The interface provides the inertia tensor $\D$, external torque estimate $\teh$ and state information $\q$, $\dot{\q}$. A schematic representation of the architecture is shown in \cref{fig:controlsetup}.

\begin{figure}[ht]
    \centering
    \includegraphics[width=.68\linewidth]{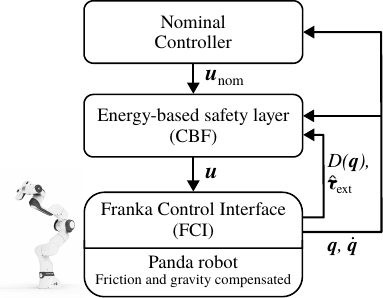}
    \caption{Control architecture of the experimental setup.}
    \label{fig:controlsetup}
\end{figure}

To reduce the effect of sensor noise on the velocity estimates, we use a discrete joint acceleration rate limiter:
\begin{equation*}
    \dot{\q}_k = \dot{\q}_{k-1} + \min(\max(\dot{\q}_{k}-\dot{\q}_{k-1}, -\Delta_\mathrm{t} \ddot{\q}_\text{max}), \Delta_\mathrm{t} \ddot{\q}_\text{max}),
\end{equation*}
where the maximum joint acceleration $\ddot{\q}_\text{max}$ is set to the robot's documented limits plus a 20\% margin. $\Delta_\mathrm{t}$ denotes the time interval between consecutive measurements $\dot{\q}_{k-1}$ and $\dot{\q}_k$. The benefit of an acceleration saturation filter is that it does not introduce delay and enforces an upper bound on the noise amplitude without attenuating the signal itself.

Despite the existence of an analytical solution \eqref{eq:closedFormCBF} to the QP \eqref{eq:CBFcontrol}, our implementation leverages the OSQP quadratic program solver \cite{Stellato2020}. The reason is that the analytical solution does not allow including input saturation limits or stacking of additional (CBF) constraints. This allows straightforward extension to a more elaborate safety filter in the future.
The solver reaches sufficient convergence to the analytical solution well within the 1\,ms sample time.

\subsection{Experiment 1: Step response}
In this experiment, the equilibrium setpoint of the nominal impedance controller is moved by 40\,cm in the (horizontal) y-direction by a square wave signal. As a result, it will attempt to inject a significant amount of virtual potential energy into the physical robot. When the safety filter is active, the kinetic energy limit $\Km=1$\,J. We repeat the experiment for $\gamma \in \lbrace 1,2,10,50 \rbrace$, and with the CBF disabled.

The end-effector trajectories (y-position) are shown in \cref{fig:step_EEposition}, with the corresponding kinetic energy in \cref{fig:step_TotalEnergy}. The latter shows that the CBF effectively limits the kinetic energy, becoming more conservative with lower values of $\gamma$. In contrast, for the case without safety filter, the kinetic energy reaches up to 2.3\,J.

\begin{figure}[ht]
	\centering
	\includegraphics[width=\figwidth]{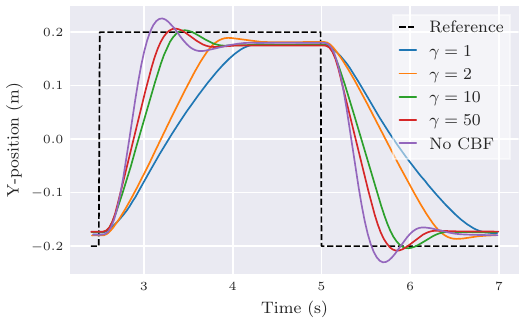}
	\caption{Experiment 1 (Step response): End-effector y-position.}
	\label{fig:step_EEposition}
\end{figure}

\begin{figure}[ht]
	\centering
	\includegraphics[width=\figwidth]{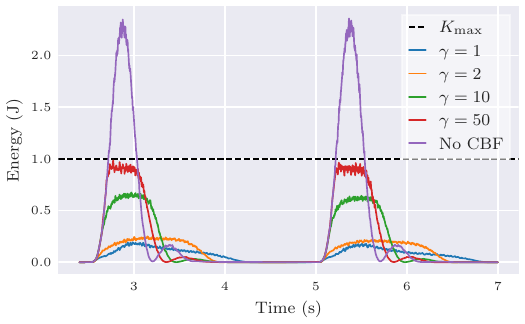}
	\caption{Experiment 1 (Step response): Total kinetic energy.}
	\label{fig:step_TotalEnergy}
\end{figure}

\cref{fig:step_PowerInjection} shows the power input of the safety filter, given by $\dot{\q}^\top (\ud - \u)$. For all experiments, total safety filter power can be observed to be non-positive, demonstrating that the safety filter only applies damping to the system as predicted by \cref{th:dampingonly}. Notice that for an individual joint the injected power may indeed be positive, however the total power input is always non-positive.  \cref{fig:step_torqueplot} shows the associated joint control torques for the experiment with $\gamma=50$. The commanded input $\u$ is identical to the desired input $\ud$ until intervention is necessary, demonstrating that the safety filter is minimally invasive.

\begin{figure}[ht]
	\centering
	\includegraphics[width=\figwidth]{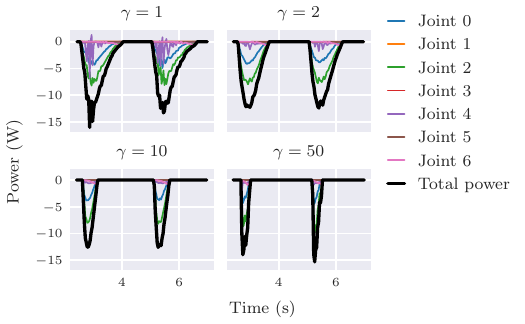}
	\caption{Experiment 1 (Step response): Safety filter power injection.}
	\label{fig:step_PowerInjection}
\end{figure}

\begin{figure}[ht]
	\centering
	\includegraphics[width=\figwidth]{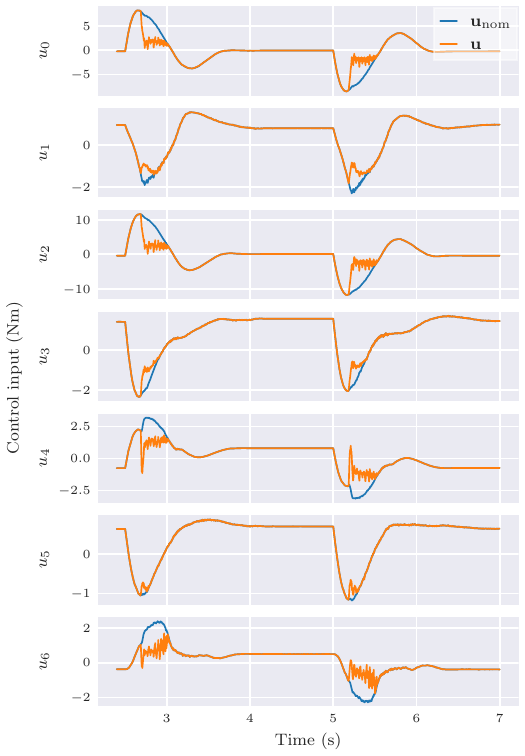}
	\caption{Experiment 1 (Step response): Nominal desired control action $\ud$ compared to the filtered control action $\u$ ($\gamma=50$).}
	\label{fig:step_torqueplot}
\end{figure}



\subsection{Experiment 2: Contact loss}
In this experiment, a string is attached to the end-effector, fixed to the base on the other end, and then brought under 50\,N of tension by lifting the equilibrium setpoint of the Cartesian impedance controller up by 25\,cm, resulting in approximately 6.25\,J of stored energy in the virtual spring. The fixed end of the string is then suddenly released, allowing the robot to accelerate upwards. This experiment is again repeated for $\gamma\in\{1,2,10,50\}$ and without the CBF.

The resulting end-effector motion is shown in \cref{fig:Franka_slideshow_horizontal}, as snapshots of the experiment without and with the safety filter using $\gamma=50$.
Upon release of the string, the stored control energy is released and the end-effector rapidly moves up towards the equilibrium. This is similar to the robot slipping off a surface it is pushing against in a sudden loss-of-contact scenario. The total kinetic energy is shown in \cref{fig:tension_TotalEnergy} and the CBF-induced power injection is shown in \cref{fig:tension_PowerInjection}, both for all values of $\gamma$ respectively. The former shows that approx. 1.7~J of kinetic energy is injected by the nominal controller without safety filter, and that the excess is effectively dissipated when the safety filter is activated.
For $\gamma=50$ the energy limit is momentarily exceeded, which we attribute to limited torque tracking capability of the robot's actuators. However, this breach is small, and cases with more conservative values of $\gamma$ remain far from the boundary, suggesting that $\gamma=50$ might be slightly too high for the capabilities of this system. \cref{fig:tension_PowerInjection} shows the total safety filter power input, which is negative for all experiments as expected due to its damping nature. 


\begin{figure}[ht]
	\centering
	\includegraphics[width=\figwidth]{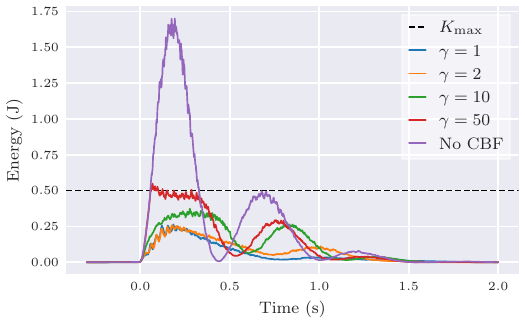}
	\caption{Experiment 2 (Contact loss): Total kinetic energy.}
	\label{fig:tension_TotalEnergy}
\end{figure}

\begin{figure}[ht]
	\centering
	\includegraphics[width=\figwidth]{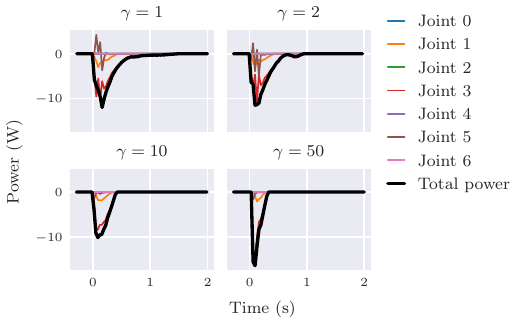}
	\caption{Experiment 2 (Contact loss): Safety filter power.}
	\label{fig:tension_PowerInjection}
\end{figure}



\begin{figure*}[!t]
    \centering
    \includegraphics[width=0.8\linewidth]{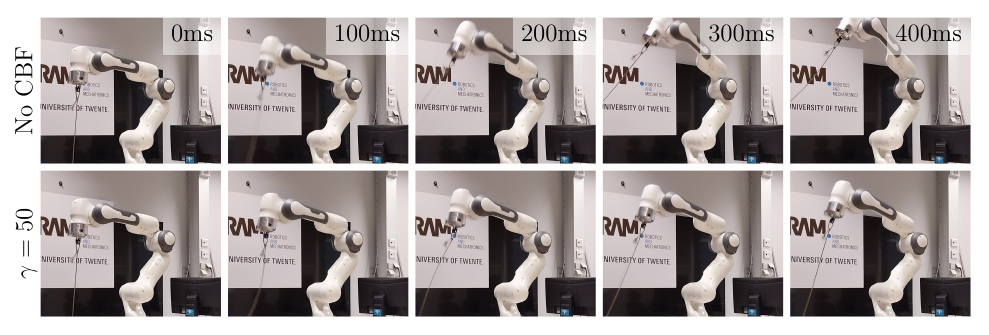}
    \caption{Experiment 2 (Contact loss): Snapshots of the experiment. Top: No safety filter. Bottom: Safety filter with $\gamma=50$.}
    \label{fig:Franka_slideshow_horizontal}
\end{figure*}

\subsection{Experiment 3: External interaction}
In the external interaction experiment, we disable the nominal controller (i.e., $\ud = \mathbf{0}$) and subject the robot to an unmodelled external power input, by physically pushing the end-effector by hand.
We compare three different cases: 
\begin{enumerate}
    \item Without safety filter;
    \item With interaction-agnostic safety filter (\cref{eq:ECBFconstraint});
    \item With interaction-aware safety filter (\cref{eq:ECBFconstraintinteractive}).
\end{enumerate}
In all cases the kinetic energy limit is set to $\Km=0.3$\,J and $\gamma=50$ for the safety filter.

The kinetic energy for all three cases is shown in \cref{fig:push_TotalEnergy}. \cref{fig:push_PowerFlow} shows the relevant corresponding power flows: 1) power injected by the operator, 2) power injected by the safety filter, and 3) their sum, which is the net power input into the system. Although the three experiments are not identical, the maximum operator power input is of comparable magnitude and duration. The red zones in \cref{fig:push_PowerFlow} indicate when the kinetic energy limit is exceeded (as per \cref{fig:push_TotalEnergy}).

Considering first \cref{fig:push_TotalEnergy}, we observe that both the interaction-agnostic and interaction-aware safety filters decrease the kinetic energy error compared to the experiment without safety filter. Critically, \cref{fig:push_TotalEnergy} shows that incorporating the estimate of the external power input reduces the error to near the kinetic energy limit, even if it is still momentarily exceeded. We attribute the latter to the relatively poor quality of the external torque estimate $\teh$.

Now considering \cref{fig:push_PowerFlow}, we observe that the safety filter produces a larger negative power when provided with the estimate of external power input. This is consistent with the reduced kinetic energy error observed in \cref{fig:push_TotalEnergy}.


\begin{figure}[ht]
	\centering
	\includegraphics[width=\figwidth]{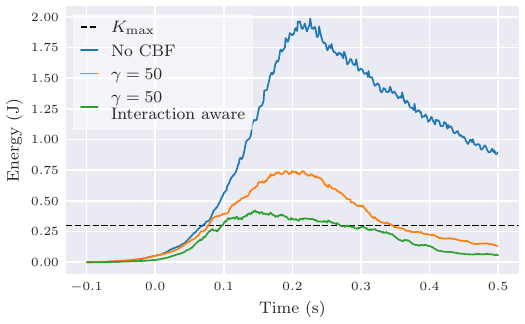}
	\caption{Experiment 3 (External interaction): Total kinetic energy.}
	\label{fig:push_TotalEnergy}
\end{figure}

\begin{figure}[ht]
	\centering
	\includegraphics[width=\figwidth]{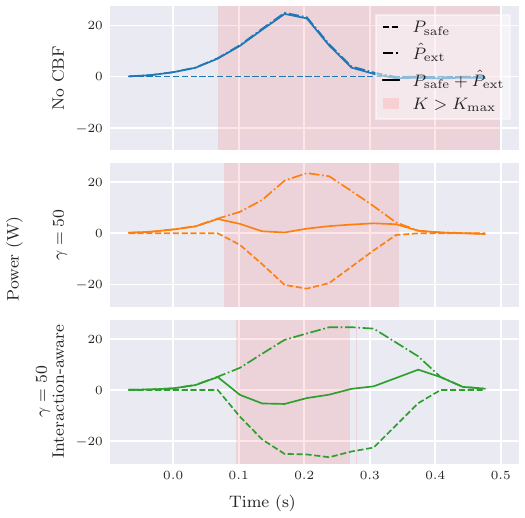}
	\caption{Experiment 3 (External interaction): Internal and external power injection for $\gamma=50$. The red zone indicates where the robot exceeds its kinetic energy limit.}
	\label{fig:push_PowerFlow}
\end{figure}


\subsection{Experiment 4: Kinetic energy error validation}
With this final experiment we aim to validate the kinetic energy error ($\Ke - \Km$) as predicted by \cref{th:povershoot}. This requires precise knowledge of the unmodelled power input, beyond what can be achieved through the robot's external torque estimation $\teh$ which can be of poor quality especially during dynamic motions. Hence, we achieve this by adding an additional term $\boldsymbol{u}_\mathrm{err}$ to the torque inputs generated by the CBF-based safety filter, as shown in \cref{fig:controlsetup_unmodelled_input}.

\begin{figure}[ht]
    \centering
    \includegraphics[width=.7\columnwidth]{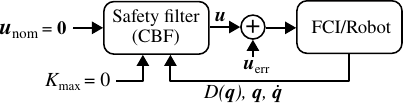}
    \caption{Experiment 4: Unmodelled power input through $\u_\mathrm{err}$.}
    \label{fig:controlsetup_unmodelled_input}
\end{figure}

The unmodelled power is injected by applying a virtual force horizontally in the y-direction at the end-effector. After an initial push to initiate motion, this force is regulated with velocity to provide a constant power input $P_\mathrm{ext}$:
\begin{align}
    \boldsymbol{f}_\mathrm{ee}&=
    \begin{pmatrix}
        0 & P_\mathrm{ext}/\dot{x}_\mathrm{ee,y}  & 0
    \end{pmatrix},\\
    \u_\mathrm{err} &= \boldsymbol{J}^\top(\q) \boldsymbol{f}_\mathrm{ee}
\end{align}
where $\dot{x}_\mathrm{ee,y}$ denotes the y-direction component of the end-effector velocity $\dot{\x}_\mathrm{ee} = \boldsymbol{J}(\q)\dot{\q}$, and $\mathbf{f}_\mathrm{ee}$ denotes the applied virtual force.
The kinetic energy is then measured at steady-state, which occurs when the external input power and CBF-induced (damping) power are at equilibrium. We set $\Km=0$\,J, and perform the experiment for various power input values and $\gamma \in \lbrace 5,10,20,30,40,50 \rbrace$.


\cref{fig:interaction_overshootfit} shows the resulting steady-state kinetic energy (error), as function of the input power and for different values of $\gamma$. The coloured circles indicate the data, and the dashed lines show a linear least-squares fit per value of $\gamma$. The magnitude of the kinetic energy (and, as $\Km=0$\,J, its error) is lower for higher values of $\gamma$. In addition, the linear fits closely match the the linear relation predicted by \cref{eq:powerovershoot}, as the slope of each curve is approx. $\gamma^{-1}$. This confirms the prediction of \cref{th:povershoot}, that although a lower value of $\gamma$ produces more conservative behaviour regarding forward invariance of the safe set, rejection of unmodelled (external) disturbances is indeed reduced.

\begin{figure}[ht]
	\centering
	\includegraphics[width=\figwidth]{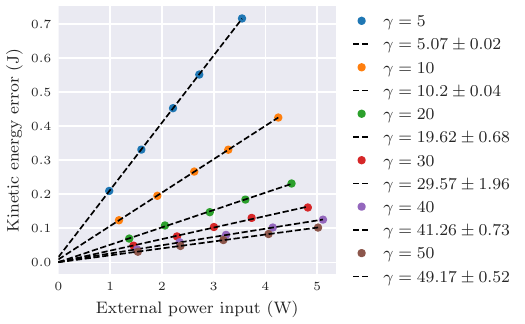}
	\caption{Experiment 4: Kinetic energy (error) versus external unmodelled power input, with linear least-squares fits.}
	\label{fig:interaction_overshootfit}
\end{figure}
\section{Discussion}
\label{sec:discussion}

The extensive experimental results presented in the previous Section demonstrate that a CBF-based safety filter is an effective approach to limiting the kinetic energy of a robot. We found in our experiments that it is easy to tune, requiring only a single parameter $\gamma$ and no knowledge of the nominal controller, and that it provides robust performance in a variety of situations.

The parameter $\gamma$ can be conveniently interpreted as a measure of conservatism, with lower values being increasingly conservative. This was clearly demonstrated by Experiments 1 and 2, in which reducing $\gamma$ kept the robot further away from the kinetic energy limit. In the case of a torque-controlled manipulator such as here, its value is practically upper bounded by the capability of the robot to achieve rapid changes in desired joint torques as the robot reaches the kinetic energy limit, as observed for the case with $\gamma=50$ in Experiment 2.

Interestingly, as shown by \cref{th:povershoot} and Experiment 4, decreasing $\gamma$ also reduces robustness against unmodelled (external) disturbances, in the sense of increased kinetic energy errors beyond the chosen limit. Experiment 3 demonstrated that incorporating an estimate of external interaction torques $\teh$ into the computation of the CBF (which we called \textit{interaction-aware}) can reduce or remove such errors, however we reiterate that it can be difficult to obtain accurate estimates of such external disturbances.

The choice of the kinetic energy limit $\Km$ itself is almost entirely task- and situation-dependent, and outside the scope of this work. Considering safety, we envision that its value would be determined by external systems, e.g. planning and/or vision systems that assess the level of danger in a given situation, such as human proximity.


\section{Conclusions and future work}
\label{sec:conclusions}
We have presented, experimentally validated, and discussed a Control Barrier Function based approach to limit the kinetic energy of torque-controlled robots. Its energetic and disturbance rejection properties were thoroughly analysed. Taking the form of a \textit{safety filter}, it requires zero knowledge of the nominal controller, which enables its use with black box (e.g., learning) controllers, providing them with strong guarantees on closed-loop energetic behaviour. Furthermore, the approach is minimally invasive, that is, the behaviour of the nominal controller is unaltered until intervention is necessary to keep the system within the safe set. These properties make such a safety filter attractive and straightforward to implement.

We are working towards extending the proposed schemes to limit not only the manipulator's total kinetic energy but also the kinetic energy transferable in specific task-space directions. In this way it will be possible to address protocols for safety hazards by restricting energy transfer in directions where human operators are present, reducing conservatism.



\bibliographystyle{unsrt}  
\bibliography{references} 

\begin{thebibliography}{10}

\bibitem{matheson2019human}
Eloise Matheson, Riccardo Minto, Emanuele~GG Zampieri, Maurizio Faccio, and Giulio Rosati.
\newblock Human--robot collaboration in manufacturing applications: A review.
\newblock {\em Robotics}, 8(4):100, 2019.

\bibitem{robotics12030084}
Carlo Weidemann, Nils Mandischer, Frederick van Kerkom, Burkhard Corves, Mathias Hüsing, Thomas Kraus, and Cyryl Garus.
\newblock Literature review on recent trends and perspectives of collaborative robotics in work 4.0.
\newblock {\em Robotics}, 12(3), 2023.

\bibitem{Robla-Gomez2017WorkingEnvironments}
S.~Robla-Gomez, Victor~M. Becerra, J.~R. Llata, E.~Gonzalez-Sarabia, C.~Torre-Ferrero, and J.~Perez-Oria.
\newblock {Working Together: A Review on Safe Human-Robot Collaboration in Industrial Environments}.
\newblock {\em IEEE Access}, 5:26754--26773, 2017.

\bibitem{Hamad2023AInteraction}
Mazin Hamad, Simone Nertinger, Robin~J. Kirschner, Luis Figueredo, Abdeldjallil Naceri, and Sami Haddadin.
\newblock {A Concise Overview of Safety Aspects in Human-Robot Interaction}.
\newblock (101017274):1--15, 2023.

\bibitem{Brunke2022SafeLearning}
Lukas Brunke, Melissa Greeff, Adam~W. Hall, Zhaocong Yuan, Siqi Zhou, Jacopo Panerati, and Angela~P. Schoellig.
\newblock {Safe Learning in Robotics: From Learning-Based Control to Safe Reinforcement Learning}.
\newblock {\em Annual Review of Control, Robotics, and Autonomous Systems}, 5(1):411--444, 5 2022.

\bibitem{InternationalOrganizationforStandardization2016Robots15066:2016}
{International Organization for Standardization}.
\newblock {Robots and robotic devices - Collaborative robots (ISO Standard No. 15066:2016)}, 2 2016.

\bibitem{Malm2019DynamicRobots}
Timo Malm, Timo Salmi, Ilari Marstio, and Jari Montonen.
\newblock {Dynamic safety system for collaboration of operators and industrial robots}.
\newblock {\em Open Engineering}, 9(1):61--71, 3 2019.

\bibitem{mansfeld_safety_2018}
Nico Mansfeld, Mazin Hamad, Marvin Becker, Antonio~Gonzales Marin, and Sami Haddadin.
\newblock Safety map: A unified representation for biomechanics impact data and robot instantaneous dynamic properties.
\newblock 3(3):1880--1887.

\bibitem{Kishi2003}
Y.~Kishi, {Zhi Wei Luo}, F.~Asano, and S.~Hosoe.
\newblock {Passive impedance control with time-varying impedance center}.
\newblock In {\em Proceedings 2003 IEEE International Symposium on Computational Intelligence in Robotics and Automation. Computational Intelligence in Robotics and Automation for the New Millennium (Cat. No.03EX694)}, pages 1207--1212. IEEE, 2003.

\bibitem{Tsetserokou2007}
Dzmitry Tsetserukou, Riichiro Tadakuma, Hiroyuki Kajimoto, Naoki Kawakami, and Susumu Tachi.
\newblock {Towards Safe Human-Robot Interaction: Joint Impedance Control of a New Teleoperated Robot Arm}.
\newblock In {\em RO-MAN 2007 - The 16th IEEE International Symposium on Robot and Human Interactive Communication}, pages 860--865. IEEE, 2007.

\bibitem{roozing_energy-based_2020}
Wesley Roozing, Stefan~S. Groothuis, and Stefano Stramigioli.
\newblock Energy-based safety in series elastic actuation.
\newblock In {\em 2020 {IEEE} International Conference on Robotics and Automation ({ICRA})}, pages 914--920. {IEEE}.

\bibitem{Haddadin2016}
Sami Haddadin and Elizabeth Croft.
\newblock {Physical Human–Robot Interaction}.
\newblock In {\em Springer Handbook of Robotics}, pages 1839--1840. 2016.

\bibitem{Tadele2014ThePublications}
Tadele~Shiferaw Tadele, Theo de~Vries, and Stefano Stramigioli.
\newblock {The Safety of Domestic Robotics: A Survey of Various Safety-Related Publications}.
\newblock {\em IEEE Robotics {\&} Automation Magazine}, 21(3):134--142, 9 2014.

\bibitem{Ames2019}
Aaron~D. Ames, Samuel Coogan, Magnus Egerstedt, Gennaro Notomista, Koushil Sreenath, and Paulo Tabuada.
\newblock {Control barrier functions: Theory and applications}.
\newblock {\em 2019 18th European Control Conference, ECC 2019}, pages 3420--3431, 2019.

\bibitem{Cortez2020}
Wenceslao~Shaw Cortez and Dimos~V. Dimarogonas.
\newblock {Correct-by-Design Control Barrier Functions for Euler-Lagrange Systems with Input Constraints}.
\newblock In {\em 2020 American Control Conference (ACC)}, pages 950--955. IEEE, 7 2020.

\bibitem{Landi2019}
Chiara~Talignani Landi, Federica Ferraguti, Silvia Costi, Marcello Bonfe, and Cristian Secchi.
\newblock {Safety barrier functions for human-robot interaction with industrial manipulators}.
\newblock In {\em 2019 18th European Control Conference, ECC 2019}, 2019.

\bibitem{Rauscher2016ConstrainedFunctions}
Manuel Rauscher, Melanie Kimmel, and Sandra Hirche.
\newblock {Constrained robot control using control barrier functions}.
\newblock In {\em 2016 IEEE/RSJ International Conference on Intelligent Robots and Systems (IROS)}, pages 279--285. IEEE, 10 2016.

\bibitem{Ferraguti2022}
Federica Ferraguti, Chiara~Talignani Landi, Andrew Singletary, Hsien-Chung Lin, Aaron Ames, Cristian Secchi, and Marcello Bonf{\`{e}}.
\newblock {Safety and Efficiency in Robotics: The Control Barrier Functions Approach}.
\newblock pages 15--30. 2022.

\bibitem{Singletary2021}
Andrew Singletary, Shishir Kolathaya, and Aaron~D. Ames.
\newblock {Safety-Critical Kinematic Control of Robotic Systems}.
\newblock {\em Proceedings of the American Control Conference}, 2021-May:14--19, 2021.

\bibitem{Capelli2022}
Beatrice Capelli, Cristian Secchi, and Lorenzo Sabattini.
\newblock {Passivity and Control Barrier Functions: Optimizing the Use of Energy}.
\newblock {\em IEEE Robotics and Automation Letters}, 7(2):1356--1363, 4 2022.

\bibitem{ortega}
R.~Ortega, A.J. Van Der~Schaft, I.~Mareels, and B.~Maschke.
\newblock Putting energy back in control.
\newblock {\em IEEE Control Systems Magazine}, 21(2):18--33, 2001.

\bibitem{ear}
Stefano Stramigioli.
\newblock Energy-aware robotics.
\newblock In M.~Kanat Camlibel, A.~Agung Julius, Ramkrishna Pasumarthy, and Jacquelien~M.A. Scherpen, editors, {\em Mathematical Control Theory I}, pages 37--50, Cham, 2015. Springer International Publishing.

\bibitem{Michel2024}
Youssef Michel, Matteo Saveriano, and Dongheui Lee.
\newblock {A Novel Safety-Aware Energy Tank Formulation Based on Control Barrier Functions}.
\newblock {\em IEEE Robotics and Automation Letters}, 9(6):5206--5213, 6 2024.

\bibitem{Califano2022OnSystems}
Federico Califano, Ramy Rashad, Cristian Secchi, and Stefano Stramigioli.
\newblock {\em {On the Use of Energy Tanks for Robotic Systems}}.
\newblock Springer International Publishing, Cham, 2023.

\bibitem{BENZI20231288}
F.~Benzi, F.~Ferraguti, and C.~Secchi.
\newblock Energy tank-based control framework for satisfying the iso/ts 15066 constraint.
\newblock {\em IFAC-PapersOnLine}, 56(2):1288--1293, 2023.
\newblock 22nd IFAC World Congress.

\bibitem{Raiola2018DevelopmentRobots}
Gennaro Raiola, Carlos~Alberto Cardenas, Tadele~Shiferaw Tadele, Theo de~Vries, and Stefano Stramigioli.
\newblock {Development of a Safety- and Energy-Aware Impedance Controller for Collaborative Robots}.
\newblock {\em IEEE Robotics and Automation Letters}, 3(2):1237--1244, 4 2018.

\bibitem{Ames2017}
Aaron~D. Ames, Xiangru Xu, Jessy~W. Grizzle, and Paulo Tabuada.
\newblock {Control Barrier Function Based Quadratic Programs for Safety Critical Systems}.
\newblock {\em IEEE Transactions on Automatic Control}, 62(8):3861--3876, 2017.

\bibitem{XU201554}
Xiangru Xu, Paulo Tabuada, Jessy~W. Grizzle, and Aaron~D. Ames.
\newblock Robustness of control barrier functions for safety critical control.
\newblock {\em IFAC-PapersOnLine}, 48(27):54--61, 2015.
\newblock Analysis and Design of Hybrid Systems ADHS.

\bibitem{Califano2023}
Federico Califano.
\newblock {Passivity-Preserving Safety-Critical Control Using Control Barrier Functions}.
\newblock {\em IEEE Control Systems Letters}, 7:1742--1747, 2023.

\bibitem{Stellato2020}
Bartolomeo Stellato, Goran Banjac, Paul Goulart, Alberto Bemporad, and Stephen Boyd.
\newblock {OSQP: an operator splitting solver for quadratic programs}.
\newblock {\em Mathematical Programming Computation}, 12(4):637--672, 12 2020.

\end{thebibliography}
\end{document}